\documentclass{article}

\usepackage{times,authblk}
\usepackage{fullpage}

\usepackage{geometry}
\usepackage{hyperref,url}
\usepackage{algorithm,algorithmic}
\usepackage{subfigure,multirow,graphicx,color,rotating}
\usepackage{amsthm,amsmath,amssymb,amsfonts}
\usepackage{enumerate}
\usepackage{verbatim,bbold}
\usepackage{mkolar_definitions}

\newcount\Comments
\definecolor{teal}{rgb}{0.3,0.8,0.8}
\newcommand{\kibitz}[2]{\ifnum\Comments=1\textcolor{#1}{#2}\fi}

\newcommand{\version}{arxiv}
\ifthenelse{\equal{\version}{arxiv}}{

}{

}

\newenvironment{packed_enum}{
\begin{enumerate}
\setlength{\itemsep}{1pt}
\setlength{\parskip}{0pt}
\setlength{\parsep}{0pt}
}{\end{enumerate}}

\begin{document}
\title{Minimax Structured Normal Means Inference}

\author[1]{Akshay Krishnamurthy
\thanks{akshaykr@cs.cmu.edu}}

\affil[1]{Microsoft Research\\
New York, NY 20011}

\maketitle

\begin{abstract}
We provide a unified treatment of a broad class of noisy structure recovery problems, known as structured normal means problems.
In this setting, the goal is to identify, from a finite collection of Gaussian distributions with different means, the distribution that produced some observed data.
Recent work has studied several special cases including sparse vectors, biclusters, and graph-based structures.
We establish nearly matching upper and lower bounds on the minimax probability of error for \emph{any} structured normal means problem, and we derive an optimality certificate for the maximum likelihood estimator, which can be applied to many instantiations.
We also consider an experimental design setting, where we generalize our minimax bounds and derive an algorithm for computing a design strategy with a certain optimality property.
We show that our results give tight minimax bounds for many structure recovery problems and consider some consequences for interactive sampling.
\end{abstract}


\section{Introduction}
\label{sec:intro}
The prevalence of high-dimensional signals in modern scientific investigation has inspired an influx of research on recovering \emph{structural information} from noisy data. 
These problems arise across a variety of scientific and engineering disciplines; for example identifying cluster structure in communication or social networks, multiple hypothesis testing in genomics, or anomaly detection in sensor networking.
Specific structural assumptions include sparsity~\cite{haupt2011distilled}, low-rankedness~\cite{donoho2014minimax}, cluster structure~\cite{kolar2011minimax}, and many others~\cite{chandrasekaran2012convex}.

The literature in this direction focuses on three inference goals: detection, localization or recovery, and estimation or denoising.
Detection tasks involve deciding whether an observation contains some meaningful information or is simply ambient noise, while recovery and estimation tasks involve more precisely characterizing the information contained in a signal.
These problems are closely related, but also exhibit important differences, and this paper focuses on the recovery problem, where the goal is to identify, from a finite collection of signals, which signal produced the observed data.

One frustration among researchers is that algorithmic and analytic techniques for these problems differ significantly for different structural assumptions. 
This issue was recently resolved in the context of the estimation, where the \emph{atomic norm}~\cite{chandrasekaran2012convex} has provided a unifying algorithmic and analytical framework, but no such theory is available for detection and recovery problems.
In this paper, we provide a unification for the recovery problem, leading to deeper understanding of how signal structure affects statistical performance.

Modern measurement technology also often provides flexibility in designing strategies for data acquisition, and this adds an element of complexity to inference tasks. 
Data acquisition by both interactive and non-interactive experimental design is the typical situation in domains ranging from network tomography to crowdsourcing, but the statistical implications of these techniques are not fully understood.
This paper also considers the experimental design setting and provides a generic solution and analysis of non-interactive experimental design for structure recovery problems.

To concretely describe our main contributions, we now develop the decision-theoretic framework of this paper.
We study the \textbf{structured normal means problem} defined by a finite collection of vectors $\Vcal = \{v_j\}_{j=1}^M \subset \RR^d$ that index a family of probability distributions $\PP_j = \Ncal(v_j, I_d)$. 
An estimator $T$ for the family $\Vcal$  is a measurable function from $\RR^d$ to $[M]$, and its maximum risk is:
\begin{align*}
\Rcal(T, \Vcal) = \sup_{j \in [M]} \Rcal_j(T, \Vcal), \qquad \Rcal_j(T, \Vcal) = \PP_j[T(y) \ne j],
\end{align*}
where we always use $y \sim \PP_j$ to be the observation. 
We are interested in the \textbf{minimax risk}:
\begin{align}
\Rcal(\Vcal) = \inf_{T} \Rcal(T,\Vcal) = \inf_{T} \sup_{j \in [M]} \PP_j[T(y) \ne j].
\label{eq:risk}
\end{align}
We call this the \textbf{isotropic setting} because each gaussian has spherical covariance.
We are specifically interested in understanding how the family $\Vcal$ influences the minimax risk.
This setting encompasses recent work on sparsity recovery~\cite{haupt2011distilled}, biclustering~\cite{kolar2011minimax,butucea2013detection}, and many graph-based problems~\cite{tanczos2013adaptive}.
An important example to keep in mind is the $k$-sets problem, where the collection $\Vcal$ is formed by vectors $\mu \mathbf{1}_{S}$ for subsets $S \subset [d]$ of size $k$ and some signal strength parameter $\mu$. 
Instantiation of our results to this example will determine the critical scaling of $\mu$ in terms of $d$ and $k$ that is necessary and sufficient for achieving asymptotically zero minimax risk. 

In the \textbf{experimental design setting}, the statistician specifies a sensing strategy, defined by a vector $B \in \RR^d_+$.
Using this strategy, under $\PP_j$, the observation is, for each $i \in [d]$:
\begin{align}
y(i) \sim v_{j}(i) + B(i)^{-1/2}\Ncal(0,1) = \Ncal(v_j(i), B(i)^{-1}).
\label{eq:nonunif_observation}
\end{align}
If $B(i) = 0$, then we say that $y(i) = 0$ almost surely.
We call this distribution $\PP_{j,B}$, to denote the dependence both on the target signal $v_j$ and the sensing strategy $B$. 
The total measurement effort, or \emph{budget}, used by the strategy is $\|B\|_1$, and we are interested in signal recovery under a budget constraint.
Specifically, the minimax risk in this setting is:
\begin{align}
\Rcal(\Vcal, \tau) = \inf_{T, B: \|B\|_1 \le \tau} \sup_{j \in [M]} \PP_{j,B}[T(y) \ne j].
\end{align}

With this background, we now state our main contributions:
\begin{enumerate}
\item We give nearly matching upper and lower bounds on the minimax risk for both isotropic and experimental design settings (Theorems~\ref{thm:gaussian_minimax_bd} and~\ref{thm:gaussian_noninteractive_bd}). 
  This result matches many special cases that we are aware of~\cite{tanczos2013adaptive}. 
  Moreover, in examples with an asymptotic flavor, including the $k$-sets example, this shows that the maximum likelihood estimator (MLE) achieves the minimax rate. 
\item In the isotropic case, we derive a condition on the family $\Vcal$ under which the MLE exactly achieves the minimax risk, thereby certifying optimality of this estimator.
\item We give sufficient conditions that certify optimality of an experimental design strategy and also give an algorithm for computing such a strategy prior to data acquisition.
\item Lastly, we provide many examples to demonstrate the generality and applicability of our results. 
\end{enumerate}

We adopt the following notation: 
For a natural number $M$, we use $[M]$ to denote the set $\{1,\ldots,M\}$. 
For a sequence of problems indexed by a natural number $n \in \NN$ and a signal strength parameter $\mu$, we often state results in terms of the \emph{minimax rate} and use the notation $\mu \asymp f(n)$ to denote this asymptotic scaling. 
This notation means that if $\mu = \omega(1)f(n)$, then the minimax risk can be driven to zero, while if $\mu = o(1)f(n)$, then the minimax risk approaches one asymptotically.
Finally, for vectors $v,M \in \RR^d$, we use $\|v\|_M = \sqrt{v^T\textrm{diag}(M)v}$ to denote the Mahalanobis norm. 

\section{Related Work}
\label{sec:related}

The structured normal means problem has a rich history in statistics and recent attention has focused on combinatorial structures.
This line is motivated by statistical applications involving complex data sources, such as tasks in graph-structured signal processing~\cite{sharpnack2013near}. 
Focusing on detection problems, a number of papers study various combinatorial structures, including $k$-sets~\cite{addarioberry2010combinatorial}, cliques~\cite{tanczos2013adaptive}, paths~\cite{ariascastro2008searching}, and clusters~\cite{sharpnack2013near} in graphs. 
While there are comprehensive results for many of these examples, a unifying theory for detection problems is still undeveloped.

Turning to recovery or localization, again several specific examples have been analyzed.
The most popular example is the biclustering problem~\cite{kolar2011minimax,butucea2013detection,tanczos2013adaptive}, which we study in Section~\ref{ssec:gaussian_biclusters}.
However, apart from this example and a few others~\cite{tanczos2013adaptive}, minimax bounds for the recovery problem are largely unknown.
Moreover, we are unaware of a broadly applicable analysis like the method we develop here.

A unified treatment is possible for estimation problems, where the atomic norm framework gives sharp phase transitions on the mean squared error of the maximum likelihood estimator~\cite{chandrasekaran2012convex,amelunxen2014living,oymak2013sharp}. 
\ifthenelse{\equal{\version}{arxiv}}{The atomic norm is a generic approach for encoding structural assumptions by decomposing the signal into a sparse convex combination of a set of base atoms (e.g., one-sparse vectors).
While this line primarily focuses on linear inverse problems~\cite{chandrasekaran2012convex,amelunxen2014living}, there are results for the estimation problem~\cite{oymak2013sharp,bhaskar2013atomic}, although neither set of results gives lower bounds on the minimax risk.}{}
Note that this approach is based on convex relaxation, and it is not immediate that such a relaxation will succeed for the recovery problem, as the probability of error for any dense family is one.
Relatedly, the non-convexity of our risk poses new challenges that do not arise with the mean squared error objective.

The recovery problem we consider here has also been extensively studied in the signal processing and information theory literature, where it is referred to as Gaussian detection, or decoding with Additive White Gaussian Noise (AWGN), although the motivation and results are quite different.
As the goal in channel coding is to reliably transmit as many bits of information as possible across a noisy channel, the vast majority of channel coding results focus on \emph{codebook design}~\cite{mackay2003information}.
Researchers have studied structured (random and non-random) codebooks solely for computational efficiency, as random codes achieve optimal transmission rates but lead to a computationally intractable decoding problem. 
In contrast, in our setting the structured codebook is inherent to the problem and the main object of interest; the analyst has no control over the codebook and wants to achieve optimal decoding performance for the codebook specified. 

Nevertheless, one line of work from this community analyzes the error probability for the maximum likelihood estimator/decoder for a given codebook (See~\cite{sason2006performance} for a survey).
Classical upper bounds include the min-distance bound and the Gallager bound~\cite{gallager1963low}, although the bound we prove here is also well-known~\cite{herzberg1994techniques}.
Lower bounds come in two flavors: (a) sphere-packing lower bounds and (b) lower bounds on the maximum-likelihood error probability.
The former is a lower bound that is independent of the particular codebook, so it does not give tight bounds on any specific family of vectors, while the latter applies only to the MLE, so it does not relate to the minimax risk.
In contrast, our technique simultaneously applies to any codebook and any estimator, leading to precise lower bounds on the minimax risk. 
To our knowledge, apart from the upper bound in Theorem~\ref{thm:gaussian_minimax_bd}, the results proved here do not appear in the information theory literature. 

Turning briefly to the experimental design setting, a number of recent advances aim to quantify the statistical improvements enabled by experimental or interactive design in specific normal means instantiations~\cite{haupt2011distilled,tanczos2013adaptive}.
\ifthenelse{\equal{\version}{arxiv}}{A unifying, interactive algorithm was recently proposed in the bandit optimization setting~\cite{chen2014combinatorial} but it is not known to improve on non-interactive approaches for the setting we consider.}{}
This work makes important progress, but a general-purpose interactive algorithm and a satisfactory characterization of the advantages offered by interactive sampling remain elusive open questions.
This paper makes progress on the latter by developing lower bounds against all non-interactive approaches.

\section{Main Results}
\label{sec:results}
In this section we develop the main results of the paper.
\ifthenelse{\equal{\version}{arxiv}}{We start by bounding the minimax risk in the isotropic setting, then develop a certificate of optimality for the maximum likelihood estimator.
Lastly, we turn to the experimental design setting.
We provide proofs in Appendix~\ref{sec:gaussian_proofs}.}{We provide proofs in supplementary material~\cite{krishnamurthy2015minimaxity}.}

\subsection{Bounds on the Isotropic Minimax Risk}
In the isotropic case, recall that we are given a finite collection $\Vcal$ of vectors $\{v_j\}_{j=1}^M$ and an observation $y \sim \Ncal(v_j, I_d)$ for some $j \in [M]$.
Given such an observation, a natural estimator is the maximum likelihood estimator (MLE), which outputs the index $j$ for which the observation was most likely to have come from.
This estimator is defined as:
\begin{align}
T_{\textrm{MLE}}(y) &= \argmax_{j\in [M]} \PP_j(y) = \argmin_{j \in [M]} \|v_j - y\|_2^2.
\end{align}
We will analyze this estimator, which partitions $\RR^d$ based on a Voronoi Tessellation of the set $\Vcal$. 

\ifthenelse{\equal{\version}{arxiv}}{As stated, the running time of the estimator is $O(Md)$, but it is worth pausing to remark briefly about computational considerations.
In many examples of interest, the class $\Vcal$ is combinatorial in nature, so $M$ may be exponentially large, and efficient implementations of the MLE may not exist.
However, as our setup does not preclude unstructured problems, the input to the estimator is the complete collection $\Vcal$, so the running time of the MLE is linear in the input size.
If the particular problem is such that $\Vcal$ can be compactly represented (e.g. it has combinatorial structure), then the estimator may not be polynomial-time computable.
This presents a real issue, as researchers have shown that a minimax-optimal polynomial time estimator is unlikely to exist for the biclustering problem~\cite{chen2014statistical,ma2013computational}, which we study in Section~\ref{sec:gaussian_examples}.
However, since the primary interest of this work is statistical in nature, we will ignore computational considerations for most of our discussion. }{}

Our first result is a generic characterization of the minimax risk, which involves analysis of the MLE.
The following function, which we call the \textbf{Exponentiated Distance Function}, plays a fundamental role.
\begin{definition}
For a family $\Vcal$ and $\alpha > 0$, the \textbf{Exponentiated Distance Function} (EDF) is:
\begin{align}
W(\Vcal, \alpha) &= \max_{j \in [M]} W_j(\Vcal, \alpha)\\
 \textrm{with} \qquad W_j(\Vcal, \alpha) &= \sum_{k \ne j} \exp\left( \frac{-\|v_j - v_k\|_2^2}{\alpha}\right)
\end{align}
\end{definition}
In the following theorem, we show that the EDF governs the performance of $T_{\textrm{MLE}}$.
More importantly, this function also leads to a lower bound on the minimax risk, and the combination of these two statements shows that the MLE is nearly optimal for \emph{any} structured normal means problem. 

\begin{theorem}
Fix $\delta \in (0,1)$. If $W(\Vcal, 8) \le \delta$, then $\Rcal(\Vcal) \le \Rcal(\Vcal, T_{\textrm{MLE}}) \le \delta$.
On the other hand, if $W(\Vcal, 2(1-\delta)) \ge 2^{\frac{1}{1-\delta}} - 1$, then $\Rcal(\Vcal) \ge \delta$. 
\label{thm:gaussian_minimax_bd}
\end{theorem}

By setting $\delta = 1/2$ above, the second statement in the theorem may be replaced by: If $W(\Vcal, 1) \ge 3$, then $\Rcal(\Vcal) \ge 1/2$. 
This setting often aids interpretability of the lower bound.
Notice that the value of $\alpha$ disagrees between the lower and upper bounds, and this leads to a gap between the necessary and sufficient conditions. 
This is not purely an artifact of our analysis, as there are many examples where the MLE does not exactly achieve the minimax risk.

However, most structured normal means problems of interest have an asymptotic flavor, specified by a sequence of problems $\Vcal_1, \Vcal_2, \ldots$, and a signal-strength parameter $\mu$, with observation $y \sim \Ncal(\mu v_j, I_d)$ for some signal $v_j$ in the current family.
In this asymptotic framework, we are interested in how $\mu$ scales with the sequence to drive the minimax risk to one or zero. 
Almost all existing examples in the literature are of this form~\cite{tanczos2013adaptive}, and in all such problems, Theorem~\ref{thm:gaussian_minimax_bd} shows that the MLE achieves the minimax rate.
\ifthenelse{\equal{\version}{arxiv}}{ To our knowledge, such a comprehensive characterization of recovery problems is entirely new.}{}

Note that the quantity $\Rcal(\Vcal,T_{\textrm{MLE}})$ is simply the worst case probability of error for the MLE, which has been extensively studied in the information theory community.
Classical upper bounds on this quantity include the min-distance bound and Gallager's bound~\cite{gallager1963low}, but the EDF-based bound here has also appeared in the literature~\cite{herzberg1994techniques}.
It is well known that the min-distance bound is often extremely loose (see Section~\ref{sec:gaussian_examples}), while application of Gallager's bound often involves challenging calculations~\cite{sason2006performance}. 
The main novelty in our result is the lower bound, which shows that the EDF also leads to a lower bound on the error probability for all estimators/decoders, generically for all families $\Vcal$. 
This new lower bound, accompanied with the existing upper bound, establishes that the EDF is the fundamental quantity in characterizing the minimax risk in these recovery problems. 

Application of Theorem~\ref{thm:gaussian_minimax_bd} to instantiations of the structured normal means problem requires bounding the EDF, which is significantly simpler than the typical derivation of this style of result.
In particular, proving a lower bound no longer requires construction of a specialized subfamily of $\Vcal$ as was the de facto standard in this line of work~\cite{kolar2011minimax,tanczos2013adaptive}.
In Section~\ref{sec:gaussian_examples}, we show how simple calculations can recover existing results.
We also show how the EDF-based bound often gives much sharper results than the min-distance bound. 

\ifthenelse{\equal{\version}{arxiv}}{Turning to the proof, the EDF arises naturally as an upper bound on the failure probability of the MLE after applying a union bound and a Gaussian tail bound.
Obtaining a lower bound based on the EDF is more challenging, and our proof is based on application of Fano's Inequality. 
We use a version of Fano's Inequality that allows a non-uniform prior and explicitly construct this prior using the EDF.
This leads to our more general lower bound. }{}

\subsection{Minimax-Optimal Recovery}
\label{ssec:gaussian_minimax_opt}
Theorem~\ref{thm:gaussian_minimax_bd} shows that that maximum likelihood estimator achieves rate-optimal performance for \emph{all} structured normal means recovery problems with the asymptotic flavor described.
However, in some cases the MLE does not achieve the exact minimax risk, and it is therefore not the optimal estimator.
In this section, we derive a sufficient condition for the \emph{exact} minimax optimality of the MLE.
As we will see via examples, the MLE is minimax optimal for several well-studied instantiations of the structured normal means problem, although analytically calculating the minimax risk is challenging.

The sufficient condition for optimality depends on a particular geometric structure of the family $\Vcal$:
\begin{definition}
A family $\Vcal$ is \textbf{unitarily invariant} if there exists a set of orthogonal matrices $\{R_i\}_{i=1}^N$ such that for each vector $v \in \Vcal$, the set $\{R_iv\}_{i=1}^N$ is exactly $\Vcal$. 
\end{definition}
In other words, the instance $\Vcal$ can be generated by applying the orthogonal transforms to any fixed vector in the collection.
Unitarily invariant problems exhibit high degrees of symmetry, and our next result shows that this symmetry suffices to certify optimality of the MLE.

\begin{theorem}
If $\Vcal$ is unitarily invariant, then the MLE is minimax optimal. 
\label{thm:gaussian_rotation}
\end{theorem}

Some remarks about the theorem are in order:
\begin{packed_enum}
\item This theorem reduces the question of optimality to a purely geometric characterization of  $\Vcal$ and, as we will see, many important problems are unitarily invariant.
\ifthenelse{\equal{\version}{arxiv}}{One common family of orthogonal matrices is the set of all permutation matrices on $\RR^d$.}{}
\item This result does not characterize the risk of the MLE; it only shows that no other estimator has better risk.
  Specifically, it does not provide an analytic bound that is sharper than Theorem~\ref{thm:gaussian_minimax_bd}.
  From an applied perspective, an optimality certificate for an estimator is more important than a bound on the risk as it helps govern practical decisions, although risk bounds enable theoretical comparison.
\item Lastly, the result is not asymptotic in nature but rather shows that the MLE achieves the \emph{exact} minimax risk for a fixed family $\Vcal$.
  We are not aware of any other results in the literature that certify optimality of the MLE under our measure of risk. 
\end{packed_enum}

\ifthenelse{\equal{\version}{arxiv}}{The proof of this theorem is based on showing that the point-wise risk $\Rcal_j(T,\Vcal)$ for the MLE is constant across the hypotheses $j \in [M]$. 
This argument uses the unitary invariance property and an explicitly representation of the MLE as a collection of polyhedral acceptance regions, with one region per hypothesis.
Finally, we employ a dual characterization of the minimax risk to show that if the point-wise risk functional is constant, then the MLE is minimax-optimal.}{}

In problems where Theorem~\ref{thm:gaussian_rotation} can be applied, our results give a complete understanding of the isotropic case.
We know that the MLE exactly achieves the minimax risk and Theorem~\ref{thm:gaussian_minimax_bd} also gives upper and lower bounds that match asymptotically.

\subsection{The Experimental Design Setting}
We now turn to the experimental design setting, where the statistician specifies a strategy $B \in \RR^{d}_+$ and receives observation $y \sim \PP_{j,B}$ given by Equation~\ref{eq:nonunif_observation}.
Our main insight is that the choice of $B$ only changes the metric structure of $\RR^d$, and this change can be incorporated into the proof of Theorem~\ref{thm:gaussian_minimax_bd}.
Specifically, the likelihood for hypothesis $j$, under sampling strategy $B$ is:
\begin{align*}
\PP_j(y|B) = \prod_{i=1}^d \sqrt{\dfrac{B(i)}{2\pi}} \exp\left\{\frac{-B(i) (v_j(i) - y(i))^2}{2}\right\}
\end{align*}
and the maximum likelihood estimator is:
\begin{align*}
T_{\textrm{MLE}}(y,B) = \argmin_{j \in [M]} \|v_j - y\|_B^2.
\end{align*}

We port Theorem~\ref{thm:gaussian_minimax_bd} to this setting and show the following:
\begin{theorem}
Fix $\delta \in (0,1)$ and any sampling strategy $B$ with $\|B\|_1 \le \tau$. Define the \textbf{Sampling Exponentiated Distance Function} SEDF:
\begin{align}
W(\Vcal, \alpha, B) = \max_{j\in [M]} \sum_{k \ne j} \exp\left(\frac{-\|v_j - v_k\|_{B}^2}{\alpha}\right)
\end{align}
If $W(\Vcal, 8, B) \le \delta$ then $\Rcal(\Vcal,\tau) \le \Rcal(\Vcal,T_{\textrm{MLE}}(y,B)) \le \delta$.
Conversely, if $W(\Vcal, 2(1-\delta), B) \ge 2^{\frac{1}{1-\delta}} - 1$, then $\inf_{T} \sup_{j \in [M]} \PP_{j,B}[T(y) \ne j] \ge \delta$. 
\label{thm:gaussian_noninteractive_bd}
\end{theorem}

The structure of the theorem is almost identical to that of Theorem~\ref{thm:gaussian_minimax_bd}, but it is worth making some important observations.
First, the theorem holds for any non-interactive sampling strategy $B \in \RR^d_+$, so the upper bound is strictly more general than Theorem~\ref{thm:gaussian_minimax_bd}.
\ifthenelse{\equal{\version}{arxiv}}{Therefore, this theorem also gives bounds for the non-isotropic or heteroskedastic case with known, shared covariance.}{}
Secondly, any strategy can be used to derive an upper bound on the minimax risk, but the same is not true for the lower bound.
Instead the lower bound is dependent on the strategy, so one must minimize over strategies to lower bound $\Rcal(\Vcal, \tau)$.
Fortunately, since the SEDF is convex in $B$ and the budget constraint is polyhedral, computing the best strategy can be solved by convex programming.
This gives a new algorithm for designing sampling procedures. 

\ifthenelse{\equal{\version}{arxiv}}{Specifically, for any $\alpha$, to compute an associated design strategy, we solve the convex program,
\begin{align}
\textrm{minimize}_{B \in \RR_+^d, \|B\|_1 \le \tau} \max_{j \in [M]}\sum_{k \ne j} \exp\left(\frac{-\|v_j - v_k\|_{B}^2}{\alpha}\right),
\label{eq:sampling_opt}
\end{align}
to obtain the sampling strategy $\hat{B}$ that minimizes the SEDF.
For example, solving Program~\ref{eq:sampling_opt} with $\alpha = 1$ results in a strategy $\hat{B}$, and if $W(\Vcal, 1, \hat{B}) \ge 3$, then we know that the minimax risk $\Rcal(\Vcal, \tau)$ over all strategies is at least $1/2$. 
On the other hand, solving with $\alpha = 8$ to obtain a (different) sampling strategy $\hat{B}$ and then using $\hat{B}$ with the MLE would give the tightest upper bound on the risk attainable by our proof technique. 
In Section~\ref{sec:gaussian_examples}, we demonstrate an example where this optimization leads to a non-uniform sampling strategy that outperforms uniform sampling.}{}

In general it is challenging to analytically certify that an allocation strategy $\hat{B}$ minimizes the SEDF, but in some cases it is possible. 
Since the SEDF is convex in $B$, specializing the first-order optimality conditions for the resulting convex program gives the following:
\begin{proposition}
Let $\hat{B}$ be a sampling strategy with $\|\hat{B}\|_1 = \tau$, $S(\hat{B}) \subset \Vcal$ be the set of hypotheses achieving the maximum in $W(\Vcal, \alpha, \hat{B})$, and $\pi$ be a distribution on $S(\hat{B})$.
If the quantity,
\begin{align*}
\EE_{j \sim \pi} \sum_{k \ne j} (v_k(i) - v_j(i))^2 \exp(-\|v_k - v_j\|_{\hat{B}}^2),
\end{align*}
is constant across $i \in [d]$, then $\hat{B}$ is a minimizer of $W(\Vcal, \alpha, B)$ subject to $\|B\|_1 \le \tau$.
\label{prop:optimal_sampling}
\end{proposition}

In many cases, this result leads to analytic lower bounds.
Specifically, the result is especially useful when $\hat{B}$ is uniform across the coordinates, and $S(\hat{B}) = [M]$, so that all of the hypotheses achieve the maximum.
In this case, it often suffices to choose $\pi$ to be uniform over the hypotheses and exploit the high degree of symmetry to demonstrate the condition holds.
As we will see in Section~\ref{sec:gaussian_examples}, many examples studied in the literature exhibit the requisite symmetry for this proposition to be applied in a straightforward manner.

Note that Tanczos and Castro~\cite{tanczos2013adaptive} establish a similar sufficient condition for the uniform sampling strategy to be optimal. 
Their result however is slightly less general in that it only certifies optimality for the uniform sampling strategy, whereas ours, in principle, can be applied more universally. 
In addition, their result applies only to problems where the hypotheses are of the form $\mu \mathbf{1}_S$ for a collection of subsets while ours is more general.
This generality is important for some examples in Section~\ref{sec:gaussian_examples}.
The other main difference is that their approach is not based on the SEDF, so their result is not directly applicable here. 

\section{Examples}
\label{sec:gaussian_examples}

\ifthenelse{\equal{\version}{arxiv}}{This section contains four instantiations of structured normal means problems, and concrete results easily attainable from our approach.}{This section contains three instantiations of structured normal means problems, and concrete results easily attainable from our approach.}
These examples have the asymptotic flavor described before, where we are interested in how a signal strength parameter $\mu$ scales with a sequence of instances.

The first example, the $k$-sets problem, is well studied, and as a warmup, we show how our technique recovers existing results.
The second and third examples are motivated by biclustering and hierarchical clustering applications; in both problems our techniques establish a lower bound against all non-interactive approaches, demonstrating separation between interactive and non-interactive procedures.
In the hierarchical clustering case, this separation is new. 
\ifthenelse{\equal{\version}{arxiv}}{Finally the fourth example is a graph-structured signal processing problem where we empirically show that uniform sampling can be sub-optimal.
The requisite calculations for these examples are deferred to Appendix~\ref{sec:gaussian_proofs}.}{}

\subsection{$k$-sets}
In the $k$-sets problem, we have $M = {d \choose k}$ and each vector $v_j = \mathbf{1}_{S_j}$ where $S_j \subset [d]$ and $|S_j| = k$. 
The observation is $y \sim \Ncal(\mu v_j, I_d)$ for some hypothesis $j$.
\begin{corollary}
The minimax rate for $k$-sets is $\mu \asymp \sqrt{\log(k(d-k))}$ and $\mu \asymp \sqrt{\frac{d}{\tau}\log(k(d-k))}$, with budget $\tau$.
In the isotropic case, the MLE is minimax optimal.
\label{cor:ksets}
\end{corollary}
This corollary follows simply by bounding the EDF for the $k$-sets problem using binomial approximations.
Using Proposition~\ref{prop:optimal_sampling}, it is easy to verify that uniform sampling is optimal here, which immediately gives the second claim.
Finally using the set of all permutation matrices and exploiting symmetry, we can easily verify that this class is unitarily invariant.
These bound agrees with established results in the literature~\cite{tanczos2013adaptive}.
To contrast, the min-distance bound from classical coding theory would reveal that the MLE succeeds when $\mu = \omega(\sqrt{k\log(d/k)})$.
This bound is polynomially worse than the one attained by our more-refined EDF-based bound.

\subsection{Biclusters}
\label{ssec:gaussian_biclusters}
The biclustering problem operates over $\RR^{d\times d}$ with $M = {d \choose k}^2$.
We parametrize the class $\Vcal$ with two indices so that $v_{ij} = \mathbf{1}_{S_i}\mathbf{1}_{S_j}^T \in \{0,1\}^{d\times d}$ with $|S_i| = |S_j| = k$.
The observation is $y \sim \Ncal(\mu \textrm{vec}(v_{ij}), I_{d^2})$ for a hypothesis $(i,j)$. 
\begin{corollary}
The minimax rate for biclusters is $\mu \asymp \sqrt{\frac{\log(k(d-k))}{k}}$ and $\mu \asymp \sqrt{\frac{d^2}{\tau k}\log(k(d-k))}$, with budget constraint $\tau$.
In the isotropic case, the MLE is minimax optimal.
\label{cor:biclusters}
\end{corollary}
Our bounds agree with existing analyses of this class~\cite{kolar2011minimax,butucea2013detection,tanczos2013adaptive}.
The biclustering problem is interesting because there is a simple interactive algorithm that succeeds if $\mu = \omega\left(\sqrt{\left(\frac{d^2}{\tau k^2} + \frac{d}{\tau}\right)\log d}\right)$, which is a factor of $\sqrt{k}$ smaller than the lower bound established here, demonstrating concrete statistical gains from interactivity~\cite{tanczos2013adaptive}.
\ifthenelse{\equal{\version}{arxiv}}{We provide an analysis of this interactive algorithm in Appendix~\ref{sec:gaussian_proofs}.}{We provide an analysis of this interactive algorithm in the supplementary material~\cite{krishnamurthy2015minimaxity}.}
Note also that our EDF-based bound is polynomially better than classical min-distance bound.

\subsection{Hierarchical Clustering}
We study a model for similarity-based hierarchical clustering from Balakrishnan et al.~\cite{balakrishnan2011noise}.
This model is known as the Constant Block Model (CBM) and the balanced version is parameterized by a number of objects $n$, a minimum cluster size $m$, both of which are powers of $2$, and a separation parameter $\mu$. 
The hierarchical clustering is a perfectly balanced binary hierarchy on $n$ objects with minimum cluster size $m$, where within cluster similarities are exactly $\mu$ larger than the between-cluster similarities at each level of the hierarchy.
The induced $n \times n$ similarity matrix is related to an \emph{ultrametric}
\ifthenelse{\equal{\version}{arxiv}}{(We provide a formal definition in Appendix~\ref{sec:gaussian_proofs}).}{(We provide a formal definition in~\cite{krishnamurthy2015minimaxity}).}

Balakrishnan et al.~\cite{balakrishnan2011noise} analyze a recursive spectral clustering algorithm on this model in the presence of Gaussian noise.
By associating the similarity matrix of each possible hierarchical clustering with an element of $\Vcal$, their setting is a special case of the structured normal means problem, and our results can be used to characterize the minimax rate.
\begin{corollary}
In the CBM, if $\mu = o\left(\sqrt{\frac{\log(nm)}{m}}\right)$ and under budget constraint $\tau$, if $\mu = o\left(\sqrt{\frac{n^2\log(nm)}{m\tau}}\right)$, then the minimax risk remains bounded away from $0$. 
In the isotropic case, the MLE is minimax optimal.
\label{cor:hcluster}
\end{corollary}

This corollary is a strict generalization of the minimax analysis of Balakrishnan et al.~\cite{balakrishnan2011noise} who only consider the case $m = n/2$.
Moreover, the results of Tanczos and Castro~\cite{tanczos2013adaptive} do not apply here as the family does not correspond to indicator vectors of some set system.
Thus, our more general treatment enables analysis of this important settings.

Note however that we only prove a lower bound on the minimax risk here. 
This lower bound, combined with existing analysis, establishes \emph{exponential} separation between interactive and non-interactive approaches for this hierarchical clustering setting.
The interactive algorithm of Krishnamurthy et al.~\cite{krishnamurthy2012efficient} can recover clusters of size $m = \Omega(\log^2n)$ with $\tau = \Theta(n\textrm{polylog}(n))$ and $\mu = \Theta(1)$.
On the other hand, Corollary~\ref{cor:hcluster} shows that, with these settings of $\mu$ and $\tau$, no non-interactive algorithm can recovery clusters of size $m = o(n)$ which is exponentially worse than the interactive algorithm.

\ifthenelse{\equal{\version}{arxiv}}{
\subsection{Stars}
\label{ssec:gaussian_stars}
Let $G = (V, E)$ be a graph and let the edges be numbered $1, \ldots, d$. 
The class $\Vcal$ is the set of all \emph{stars} in the graph, that is the vector $v_j \in \{0,1\}^d$ is the indicator vector of all edges emanating from the $j$th node in the graph. 
Again the observation is $y \sim \Ncal(\mu v_j, I_d)$ for some $j \in [|V|]$.
\begin{corollary}
In the stars problem if the ratio between the maximum and minimum degree is bounded by a constant, i.e. $\frac{\textrm{deg}_{\max}}{\textrm{deg}_{\min}} \le c$, then the minimax rate is $\mu \asymp \sqrt{\frac{\log(|V| - \textrm{deg}_{\min})}{\textrm{deg}_{\min}}}$. 
\label{cor:stars}
\end{corollary}
Again this agrees with a recent result of Tanczos and Castro~\cite{tanczos2013adaptive}, who consider $s$-stars of the complete graph, formed by choosing a vertex, and then activating $s$ of the edges emanating out of that vertex.
The two bounds agree in the special case of the complete graph with $s = |V|-1$, but otherwise are incomparable, as they consider different problem structures.
Note that the degree requirement here is not fundamental in Theorem~\ref{thm:gaussian_minimax_bd}, but arises from problem-specific approximations.

\begin{figure*}
\begin{center}
\includegraphics[scale=0.24]{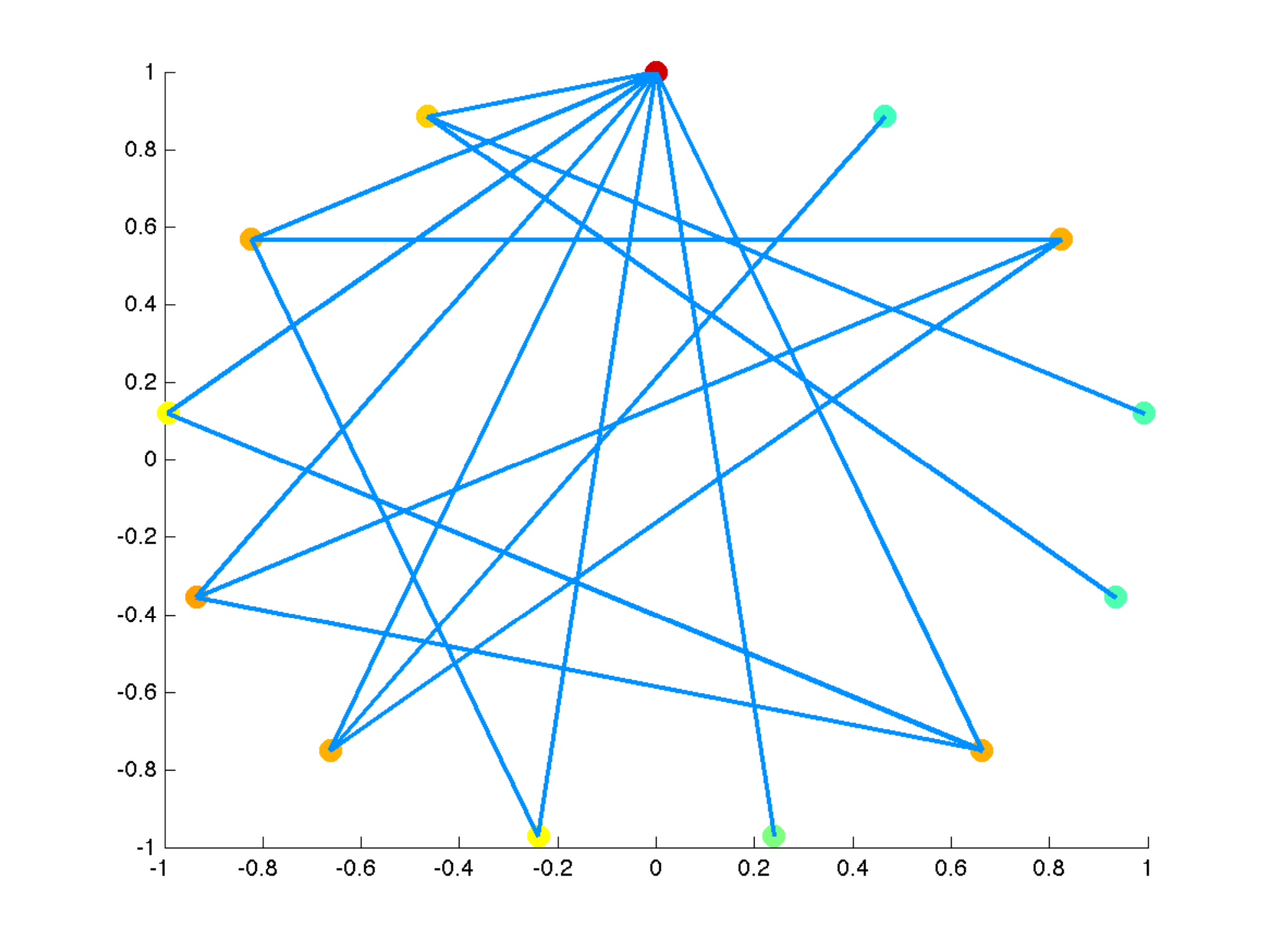}
\includegraphics[scale=0.24]{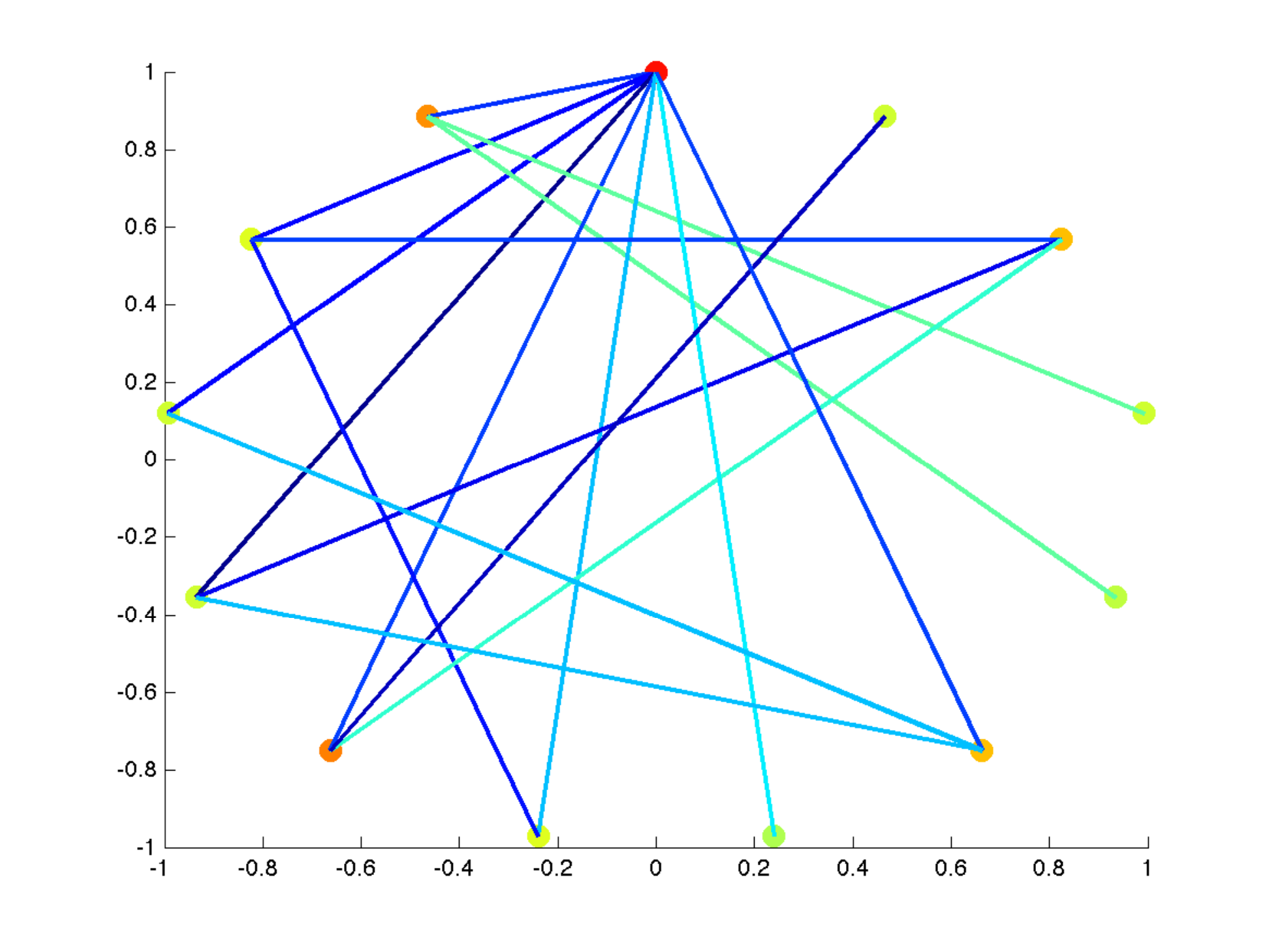}
\includegraphics[scale=0.24]{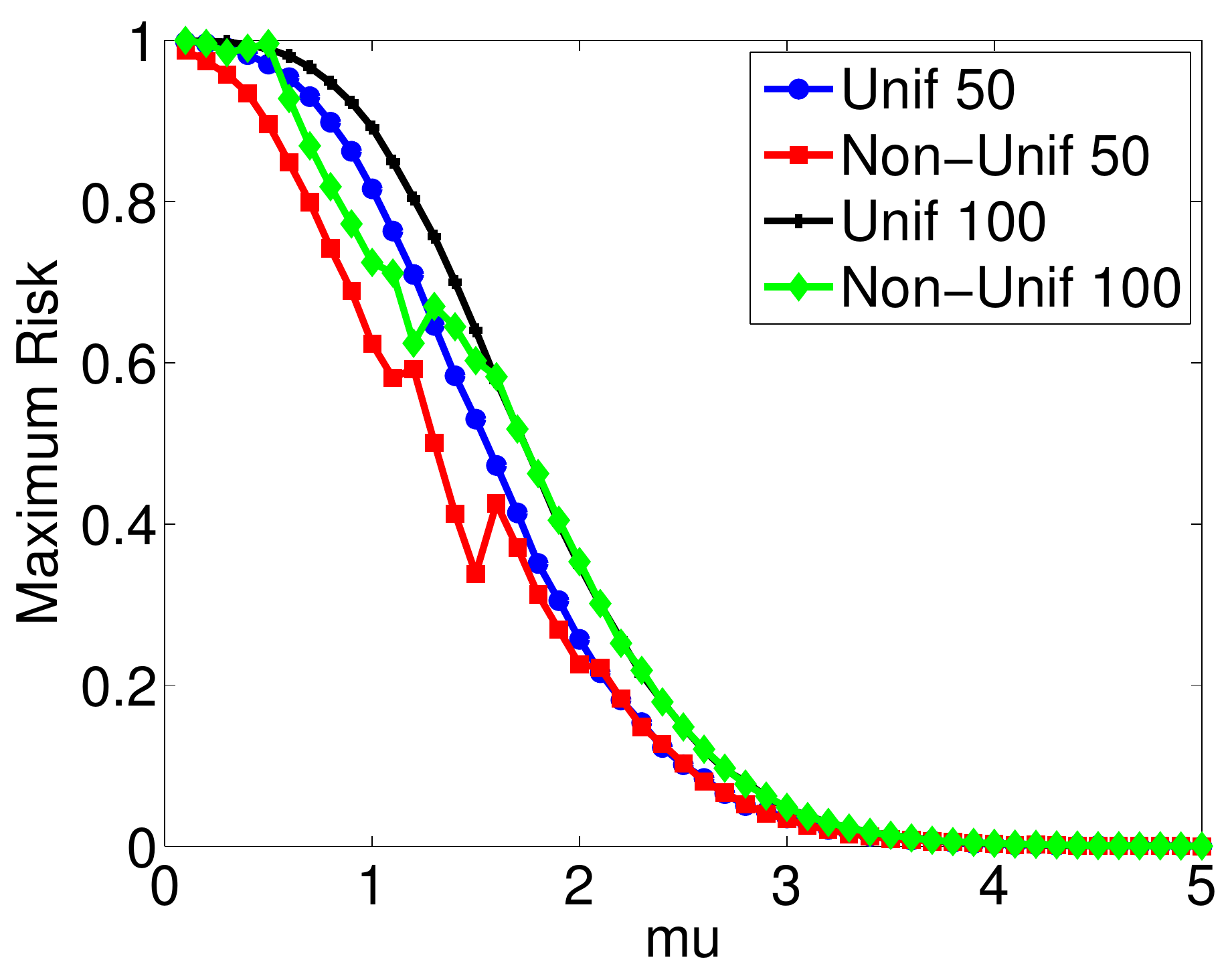}
\end{center}
\caption{Left: A realization of the stars problem for a graph with 13 vertices and 34 edges with sampling budget $\tau = 34$.
Edge color reflects allocation of sensing energy and vertex color reflects success probability for MLE under that hypothesis (warmer colors are higher for both).
Isotropic (left) has minimum success probability of $0.44$ and experimental design (center) has minimum success probability $0.56$.
Right: Maximum risk for isotropic and experimental design sampling as a function of $\mu$ for stars problem on a 50 and 100-vertex graph.}
\label{fig:gaussian_stars}
\end{figure*}

We highlight this example because the uniform allocation strategy does not necessarily minimize $W(\Vcal, \alpha, B)$.
In Figure~\ref{fig:gaussian_stars}, we construct a graph according to the Barab\'{a}si-Albert model~\cite{albert2002statistical} and consider the class of stars on this graph.
The simulation results show that optimizing the SEDF to find a sampling strategy is never worse than uniform sampling, and for low signal strengths it can lead to significantly lower maximum risk.
Note that the risk for both uniform and non-uniform sampling approaches zero as $\mu \rightarrow \infty$, so for large $\mu$, there is little advantage to optimizing the sampling scheme.
}{}

\section{Discussion}
This paper studies the structured normal means problem and gives a unified characterization of the minimax risk for isotropic and experimental design settings.
Our work provides insights into how to choose estimators and how to design sampling strategies for these problems.
Our lower bounds are critical in separating non-interactive and interactive sampling, which is an important research direction.

There are a number of exciting directions for future work, including extensions to other structure discovery problems, and to other observation models, such as compressive observations.
We are most interested in developing a unifying theory for interactive sampling, analogous to the theory developed here.
Another important and challenging direction is to consider recovery problems that involve nuisance parameters.
We look forward to studying these problems in future work.


\section*{Acknowledgements}
AK would like to thank Martin Azizyan, Sivaraman Balakrishnan, Gautam Dasarathy, James Sharpnack, and Aarti Singh for enlightening discussions during the development of this work.

\bibliography{thesis}
\bibliographystyle{plain}

\vfill
\newpage 
\appendix
\section{Proofs}
\label{sec:gaussian_proofs}

\subsection{Proof of Theorem~\ref{thm:gaussian_minimax_bd}}
\textbf{Analysis of MLE:} 
We first analyze the maximum likelihood estimator:
\begin{align*}
T_{MLE}(y) = \argmin_{j \in [M]} \|v_j - y\|_2^2
\end{align*}
This estimator succeeds as long as $\|v_k - y\|_2^2 > \|v_{j^\star} - y\|_2^2$ for each $k \ne j^\star$, when $y \sim \PP_{j^\star}$. 
This condition is equivalent to:
\begin{align*}
\|v_k - y\|_2^2 > \|v_{j^\star} - y\|_2^2 \Leftrightarrow \langle \epsilon, v_k - v_{j^\star} \rangle < \frac{1}{2}\|v_{j^\star} - v_k\|_2^2
\end{align*}
where $\epsilon \sim \Ncal(0, 1)$. 
This follows from writing $y = v_{j^\star} + \epsilon$ and then expanding the squares.
So we must simultaneously control all of these events, for fixed $j^\star$:
\begin{align*}
& \PP_{\epsilon \sim \Ncal(0, I_d)}\left[\forall k \ne j^\star. \langle \epsilon, v_k - v_{j^\star} \rangle < \|v_{j^\star} - v_k\|_2^2/2\right] \\
& = 1 - \PP_{\epsilon \sim \Ncal(0, I_d)}\left[ \exists k \ne j^\star. \langle \epsilon, v_k - v_{j^\star}\rangle \ge \|v_{j^\star} - v_k\|_2^2/2\right] \\
& \ge 1- \sum_{k \ne j^\star} \PP_{\epsilon \sim \Ncal(0, I_d)}\left[\langle \epsilon, v_k - v_{j^\star} \rangle \ge \|v_{j^\star} - v_k\|_2^2/2\right]
\end{align*}
By a gaussian tail bound, this probability is:
\begin{align*}
\PP_{\epsilon \sim \Ncal(0, I_d)}\left[ \langle \epsilon, v_k - v_{j^\star}\rangle \ge \|v_{j^\star} - v_k\|_2^2/2\right] \le \exp\left\{-\frac{1}{8}\|v_{j^\star} - v_k\|_2^2\right\}
\end{align*}
So that the total failure probability is upper bounded by:
\begin{align*}
\PP_{j^\star}[\hat{j} = j^\star] \le \sum_{k \ne j^\star} \exp\left\{-\frac{1}{8}\|v_{j^\star} - v_k\|_2^2\right\} = W_{j^\star}(\Vcal, 8)
\end{align*}
So if $j$ is the truth, then the probability of error is smaller than $\delta$ when $W_{j}(\Vcal, 8) \le \delta$.
For the maximal (over hypothesis choice $j$) probability of error to be smaller than $\delta$, it suffices to have $W(\Vcal, 8) \le \delta$. 

\textbf{Fundamental Limit:}
We now turn to the fundamental limit. 
We start with a version of Fano's inequality with non-uniform prior. 
\begin{lemma}[Non-uniform Fano Inequality]
\label{lem:clean_fano}
Let $\Theta = \{\theta\}$ be a parameter space that indexes a family of probability distributions $P_\theta$ over a space $\Xcal$.
Fix a prior distribution $\pi$, supported on $\Theta$ and consider $\theta \sim \pi$ and $X \sim P_\theta$. 
Let $f: \Xcal \rightarrow \Theta$ be any possibly randomized mapping, and let $p_e = \PP_{\theta \sim \pi, X \sim P_\theta}\left[ f(X) \ne \theta \right]$ denote the probability of error. 
Then:
\begin{align*}
p_e \ge 1 - \frac{\sum_{\theta} \pi(\theta)KL(P_\theta || P_{\pi}) + \log 2}{H(\pi)}
\end{align*}
where $P_\pi(\cdot) = \EE_{\theta \sim \pi}P_\theta(\cdot)$ is the mixture distribution.
In particular, we have:
\begin{align*}
\inf_{f} \sup_{\theta} \PP_{X \sim P_\theta}[f(X) \ne \theta] \ge \inf_{f} \EE_{\theta \sim \pi} \PP_{X \sim P_\theta}[f(X) \ne \theta] \ge 1 - \frac{\sum_{\theta} \pi(\theta)KL(P_\theta || P_{\pi}) + \log 2}{H(\pi)}
\end{align*}
\end{lemma}
\begin{proof}
Consider the Markov Chain $\theta \rightarrow X \rightarrow \hat{\theta} \triangleq f(X)$ where $\theta \sim \pi$ and $X|\theta \sim P_\theta$.
Let $E = \mathbf{1}[\hat{\theta} \ne \theta]$.
\begin{align*}
H(E | X) + H(\theta| E,X) = H(E, \theta | X) = H(\theta | X) + H(E | \theta, X) \ge H(\theta | X)
\end{align*}
Now, $H(E | \theta, X) \ge 0$ and since conditioning only reduces entropy, we have the inequality
\begin{align*}
H(\theta | X) &\le H(p_e) + H(\theta| E, X) = H(p_e) + H(\theta | E = 0, X) P[E=0] + H(\theta | E=1, X)P[E=1]\\
& = H(p_e) + p_e H(\theta)
\end{align*}
which proves the usual version of Fano's inequality.
We want to write $H(\theta|X)$ in terms of the KL divergence, using the mixture distribution $P_\pi$. 
\begin{align*}
H(\theta | X) &= H(\theta, X)- H(X) = \int \sum_{\theta} \pi(\theta) P_\theta(x) \log\left( \frac{\sum_{\theta} \pi(\theta) P_\theta(x))}{\pi(\theta)P_\theta(x)}\right)dx\\
& = \sum_{\theta} \pi(\theta) \int P_\theta(x) \log\left(\frac{P_\pi(x)}{P_\theta(x)}\right)dx  - \sum_{\theta}\pi(\theta) \log \pi(\theta)\\
& = - \sum_{\theta} \pi(\theta) KL(P_\theta || P_\pi) + H(\pi)
\end{align*}
Combining these gives the bound:
\begin{align*}
H(p_e) + p_e H(\pi) \ge H(\pi) - \sum_{\theta} \pi(\theta)KL(P_\theta || P_\pi),
\end{align*}
By upper bounding $H(p_e) \le \log 2$ and rearranging we prove the claim.
\end{proof}

For a distribution $\pi \in \Delta_{M-1}$ over the hypothesis, let $P_\pi(\cdot) = \sum_k \pi_k P_k(\cdot)$ be the mixture distribution. 
Then Fano's inequality (Lemma~\ref{lem:clean_fano}) states that the minimax probability of error is lower bounded by:
\begin{align*}
\Rcal(\Vcal) = \inf_T \sup_j \PP_j[T(y) \ne j] &\ge \inf_T \EE_{j \sim \pi} \EE_{y \sim j} \mathbf{1}[T(y) \ne j]\\
& \ge 1 - \frac{\EE_{k \sim\pi} KL(P_k || P_{\pi}) + \log 2}{H(\pi)}.
\end{align*}

Fix $\delta \in (0,1)$ and let $j^\star = \argmax_{j \in [M]} W_j(2(1-\delta))$. 
We will use a prior based on this quantity:
\begin{align*}
\pi_k \propto \exp\left(-\frac{\|v_{j^\star} - v_k\|_2^2}{2(1-\delta)}\right)
\end{align*}
With this prior, the entropy becomes:
\begin{align*}
H(\pi) &= \sum_k \pi_k \log\left( \frac{\sum_{i} \exp\left(-\frac{\|v_{j^\star} - v_i\|_2^2}{2(1-\delta)}\right)}{\exp\left(-\frac{\|v_{j^\star} - v_k\|_2^2}{2(1-\delta)}\right)}\right)\\
& = \log(W(\Vcal, 2(1-\delta)) + 1) + \sum_{k} \pi_k \frac{\|v_{j^\star} - v_k\|_2^2}{2(1-\delta)}\\
& = \log(W(\Vcal, 2(1-\delta)) + 1) + \frac{1}{1-\delta}\sum_k \pi_k KL(P_k || P_{j^\star})
\end{align*}
The $1$ inside the first $\log$ comes from the fact that in the definition $W_{j^\star}$, we do not include the term involving $j^\star$ in the sum, while our prior $\pi$ does place mass proportional to $1$ on hypothesis $j^\star$. 
The term involving the KL-divergence follows from the fact that the KL between two gaussians is one-half the $\ell_2^2$-distance between their means.

Looking at the lower bound from Fano's inequality, we see that if:
\begin{align*}
\EE_{k\sim\pi}KL(P_k || P_\pi) + \log 2 \le (1-\delta) H(\pi) = (1-\delta)\log(W(\Vcal, 2(1-\delta))+1) + \sum_k \pi_k KL(P_k || P_{j^\star})
\end{align*}
then the probability of error is lower bounded by $\delta$. 
Of course it is immediate that:
\begin{align*}
\sum_{k} \pi_k KL(P_k || P_{j^\star}) &= \sum_k \pi_k \int P_k(x) \log \left(\frac{P_k(x) P_\pi(x)}{P_\pi(x) P_{j^\star}(x)}\right)\\
& = \sum_k \pi_k \int P_k(x) \log\frac{P_k(x)}{P_\pi(x)} + \sum_k \int \pi_k P_k(x) \log\frac{P_\pi(x)}{P_{j^\star}(x)}\\
& = \sum_k\pi_k KL(P_k || P_\pi) + KL(P_\pi || P_{j^\star}) \ge \EE_k KL(P_k || P_\pi)
\end{align*}
So the condition reduces to requiring that:
\begin{align*}
\log 2 \le (1-\delta)\log(W(\Vcal, 2(\delta - 1) + 1).
\end{align*}
After some algebra, this is equivalent to:
\begin{align*}
W(\Vcal, 2(\delta-1)) \ge 2^{\frac{1}{1-\delta}} - 1 \tag*{\qed}
\end{align*}

\subsection{Proof of Theorem~\ref{thm:gaussian_noninteractive_bd}}
The proof of Theorem~\ref{thm:gaussian_noninteractive_bd} is essentially the same as the proof of Theorem~\ref{thm:gaussian_minimax_bd}, coupled with two observations.
First, for a sampling strategy $B \in \RR^{d}_+$ the maximum likelihood estimator is:
\begin{align*}
T_{\textrm{MLE}}(y,B) = \argmin_{j \in [M]} \|v_j - y\|_{B}^2
\end{align*}
so the analysis of the MLE depends on the Mahalanobis norm $\|\cdot\|_B$ instead of the $\ell_2$ norm.

Similarly, the KL divergence between the distribution $\PP_{j,B}$ and $\PP_{k,B}$ depends on the Mahalanobis norm $\|\cdot \|_B$ instead of the $\ell_2$ norm.
Specifically, we have:
\begin{align*}
KL(\PP_{j,B} || \PP_{k,B}) = \frac{1}{2}\|v_j - v_k\|_{B}^2.
\end{align*}
The lower bound proof instead use this metric structure, but the calculations are equivalent.\qed

\subsection{Proof of Proposition~\ref{prop:optimal_sampling}}
To simplify the presentation, let $f(B) = W(\Vcal, \alpha, B)$.
$f(B)$ is convex and (strictly) monotonically decreasing, so we know that the minimum will be achieved when the constraint is tight, i.e. when $\|B\|_1 = \tau$. 
The Lagrangian is:
\begin{align*}
\Lcal(B, \lambda) = f(B) + \lambda(\|B\|_1 - \tau)
\end{align*}
and the minimum is achieved at $\hat{B}$, with $\|\hat{B}\|_1 = \tau$, if there is a value $\hat{\lambda}$ such that $0 \in \partial \Lcal(\hat{B}, \hat{\lambda})$.
Observing that the subgradient is $\partial f(B) + \lambda\mathbf{1}$, it suffices to ignore the Lagrangian term and instead ensure that $\partial f(B) \propto \mathbf{1}$.
$f(B)$ is a maximum of $M$ convex functions, where $f_j(B)$ is the function corresponding to hypothesis $v_j$, and, by direct calculation, the subgradient of this function $f_j(B)$ is:
\begin{align*}
\frac{\partial f_j(B)}{\partial B_i} = \sum_{k \ne j} -(v_k(i) - v_j(i))^2 \exp(-\|v_k - v_j\|_B^2).
\end{align*}
Moreover, the subgradient of the maximum of a set of functions is the convex hull of the subgradients of all functions achieving the maximum.
This means that if there exists a distribution $\pi$, supported over the maximizers of $f(\hat{B})$, such that the expectation of the subgradients is constant, we have certified optimality of $\hat{B}$.
This is precisely the condition in the Proposition.\qed

\subsection{Proof of Theorem~\ref{thm:gaussian_rotation}}
\label{ssec:gaussian_minimax_proofs}

Our approach is based on a well-known connection between the minimax risk and the Bayes risk.
For a structured normal means problem defined by a family $\Vcal$, the \emph{Bayes risk} for an estimator $T$ under prior $\pi \in \Delta_{M-1}$ is given by:
\begin{align*}
B_\pi(T) = \sum_{j=1}^M \pi_j \PP_j[T(y) \ne j].
\end{align*}
We say that an estimator $T$ is the Bayes estimator for prior $\pi$ if it achieves the minimum Bayes risk.
A simple calculation reveals the structure of the Bayes estimator for any prior $\pi$ and this structural characterization is essential to our development.
\begin{proposition}
For any prior $\pi$, the Bayes estimator $T_\pi$ has polyhedral acceptance regions, that is the estimator is of the form:
\begin{align*}
T(Y) = j \textrm{ if } y \in A_j,
\end{align*}
with $A_j = \{x : \Gamma_j x \ge b_j\}$ and $\Gamma_j \in \RR^{M \times d}$ has $v_j - v_k$ in the $k$th row and $b_j$ has $\frac{1}{2}(\|v_j\|_2^2 - \|v_k\|_2^2) + \log \frac{\pi_k}{\pi_j}$ in the $k$th entry.
These polyhedral sets $A_j$ partition the space $\RR^d$.
\label{prop:bayes_est_structure}
\end{proposition}
\begin{proof}
To prove Proposition~\ref{prop:bayes_est_structure}, we make two claims.
First we certify that for a prior $\pi$, the Maximum a Posteriori (MAP) estimator is a Bayes estimator for prior $\pi$.
Given $\pi$, the map estimator is:
\begin{align*}
T_\pi(y) = \argmax_j \pi(j) \exp\{-\|v_j - y\|_2^2/2\}
\end{align*}
Define the posterior risk of an estimator $T$ to be the expectation of the loss, under the posterior distribution on the hypothesis.
In our case this is:
\begin{align*}
r(T|y) = \sum_{j=1}^M \mathbf{1}[T(y) \ne j] \pi(j | y) \qquad \textrm{where} \qquad \pi(j | y) \propto \pi(j) \exp\{ - \|v_j - y\|_2^2/2\}.
\end{align*}
For a fixed $y$, this quantity is minimized by letting $T(y)$ be the maximizer of the posterior, as this makes the $0-1$ loss term zero for the largest $\pi(j|y)$ value. 
Thus for each $y$ we minimize the posterior risk by letting $T(y)$ be the MAP estimate.
The result follows by the well known fact that if an estimator minimizes the posterior risk at each point, then it is the Bayes estimator. 

This argument shows that the only types of estimators we need to analyze are MAP estimators under various priors.
This gives us the requisite structure to prove Proposition~\ref{prop:bayes_est_structure}.

Specifically, for a prior $\pi$, for the MAP estimate to predict hypothesis $j$, it must be the case that:
\begin{align*}
\forall k \ne j. \qquad \pi_j \exp\{-\|v_j - y\|_2^2/2\} \ge \pi_k \exp\{-\|v_j - v_k\|_2^2/2\}.
\end{align*}
This can be simplified to:
\begin{align*}
\langle v_j - v_k, y \rangle \ge \frac{1}{2}(\|v_j\|_2^2 - \|v_k\|_2^2) + \log\frac{\pi_k}{\pi_j}.
\end{align*}
Thus the acceptance region for the hypothesis $j$ is the set of all points $y$ that satisfy all of these $M-1$ inequalities.
This is exactly the polyhedral set $A_j$.
\end{proof}

We also exploit the relationship between the minimax risk and the Bayes risk.
This is a well known result, where the prior $\pi$ below is known as the \emph{least-favorable prior}.
\begin{proposition}
Suppose that $T$ is a Bayes estimator for some prior $\pi$.
If the risk $\Rcal_j(T) = \Rcal_{j'}(T)$ for all $j \ne j' \in [M]$, then $T$ is a minimax optimal estimator.
\label{prop:bayes_minimax}
\end{proposition}
\begin{proof}
We provide a proof of this well-known result showing that the Bayes estimator with uniform risk landscape is minimax optimal. 
Let $T_\pi$ be the Bayes estimator under prior $\pi$ and let $T_0$ be some other estimator. 
Since $T_\pi$ has constant risk landscape, we know that $\max_j \Rcal_j(\Vcal, T_\pi) = B_\pi(T_\pi)$, or the minimax risk for $T_\pi$ is equal to its Bayes risk.
We know that the Bayes risk of $T_0$ is at most the minimax risk for $T_0$, i.e. $B_\pi(T_0) \le \max_j \Rcal_j(\Vcal, T_0)$.
If it were the case that $T_0$ had strictly lower minimax risk, then we have:
\begin{align*}
B_\pi(T_0) \le \max_j \Rcal_j(\Vcal, T_0) < \max_j \Rcal_j(\Vcal, T_\pi) \le B_\pi(T_\pi).
\end{align*}
However, this is a contradiction since $T_\pi$ is the Bayes estimator under prior $\pi$, meaning that it minimizes the Bayes risk. 
\end{proof}

Equipped with these results, we now turn to the proof of the theorem.

\textbf{Proof of Theorem~\ref{thm:gaussian_rotation}:}
Our goal is to apply Proposition~\ref{prop:bayes_minimax}.
By the fact that $\Rcal_j(\Vcal, T) = 1 - \PP_j[A_j]$ where $A_j$ is $T$'s acceptance region for hypothesis $j$, we must show that the $\PP_j$ probability content of the acceptance regions are constant. 
Ignoring the normalization factor of the gaussian density, this is:
\begin{align*}
\int_{A_j} \exp\{-\|v_j - x\|_2^2/2\} dx,
\end{align*}
where $A_j = \{z | \Gamma_j z \ge b_j\}$ as defined in Proposition~\ref{prop:bayes_est_structure}. 
We will exploit the unitary invariance of the family.

For any pair of hypothesis $j,k$, let $R_{jk}$ be the orthogonal matrix such that $v_k = R_{jk}v_j$ and note that $R_{kj}$, the orthogonal matrix that maps $v_k$ to $v_j$, is just $R_{jk}^T$.
This also means that $R_{jk}R_{jk}^T = R_{jk}R_{kj} = I$.
Via a change of variables $x = R_{kj}y$, the integrand becomes:
\begin{align*}
\exp\{-\|v_j - R_{kj}y\|_2^2/2\} = \exp\{-\|R_{jk}v_j - R_{jk}R_{kj}y\|_2^2/2\} = \exp\{-\|v_k - y\|_2^2/2\}.
\end{align*}
Thus, we have translated to the $P_k$ measure. 

As for the region of integration, first note that since $v_i = R_{ji}v_j$, it must be the case that $\|v_j\|_2^2 = \|v_i\|_2^2$ for all $j,i \in [M]$.
This means that the vector $b_j$ defining the acceptance region, which for the MLE has coordinates $b_j(i) = \frac{1}{2}(\|v_j\|_2^2 - \|v_i\|_2^2)$, is just the all-zeros vector. 
The region of integration is therefore:
\begin{align*}
\{z| \Gamma_j z \ge 0\} = \{z | \Gamma_jR_{kj}z \ge 0\}.
\end{align*}
We must check that this polytope is exactly $A_k$, which means that we must check that for each $i$, $(v_j - v_i)^TR_{kj}$ is a row of the $\Gamma_k$ matrix. 
But:
\begin{align*}
(v_j - v_i)^TR_{kj} = v_j^TR_{jk}^T - v_i^TR_{jk}^T = v_k^T - v_i^TR_{jk}^T.
\end{align*}
Since $v_i$ can generate the family $\Vcal$, it must be the case that $R_{jk}v_i \in \Vcal$ so that this difference does correspond to some row of $\Gamma_k$. 
Since we apply the same unitary operator to all of the rows, it must be the case that the number of distinct rows is unchanged, or in other words, there is a bijection from the rows in $\Gamma_jR_{kj}$ to the rows in $\Gamma_k$.
Therefore, the transformed region of integration, after the change of variable $x = R_{kj}y$ is exactly the acceptance region $A_k$, and the integrand is the $\PP_k$ measure.
This means that $\PP_k[A_k] = \PP_j[A_j]$ and this is true for all pairs $(j,k)$, so that the risk landscape is constant. 
By Proposition~\ref{prop:bayes_minimax}, this certifies optimality of the MLE. \qed

\subsection{Calculations for the examples}

\textbf{Calculations for $k$-Sets:}
We must upper and lower bound $W(\Vcal, \alpha)$.
First note that by symmetry, every hypothesis achieves the maximum, so it suffices to compute just one of them:
\begin{align*}
W(\Vcal, \alpha) = \sum_{k \ne j} \exp\left( -\|v_k - v_j\|_2^2/\alpha\right) = \sum_{s=1}^{k} {k \choose s} {d-k \choose s} \exp(-2s\mu^2/\alpha).
\end{align*}
This follows by noting that the $\ell_2^2$ distance between two hypothesis is the symmetric set difference between the two subsets, and then by a simple counting argument.
Using well known bounds on binomial coefficients, we obtain:
\begin{align*}
W(\Vcal, \alpha) &\le \sum_{s=1}^k \exp(s \log (ke/s) + s \log((d-k)e/s) - 2s\mu^2/\alpha)\\
& = \sum_{s=1}^k \exp(s \log(e^2k(d-k)/s^2) - 2s\mu^2/\alpha)\\
& \le k \exp(\log e^2k(d-k) - 2\mu^2/\alpha) \qquad \textrm{ if } 2\mu^2/\alpha \ge \log(e^2k(d-k))
\end{align*}
This is smaller than $\delta$ whenever $\mu^2 \ge \alpha \log(ek(d-k)/\delta)$, which subsumes the requirement above. 
For the lower bound:
\begin{align*}
W(\Vcal, \alpha) & \ge \sum_{s=1}^k \exp(s \log(k/s) + s\log((d-k)/s) - 2s\mu^2/\alpha) \ge \exp(-2\mu^2/\alpha + \log(k(d-k)))
\end{align*}
which goes to infinity if $\mu^2 = o(\alpha\log(k(d-k)))$. 

To certify that the uniform allocation strategy minimizes $W(\Vcal, \alpha, B)$, we apply Proposition~\ref{prop:optimal_sampling}. 
Fix $\tau$ and let $\hat{B}$ be such that $\hat{B}(i) = \tau/d$. 
By symmetry, every hypothesis achieves the maximum under this allocation strategy, and we will take $\pi$ to be the uniform distribution over all hypothesis.

For a hypothesis $j$ and a coordinate $i$, the subgradient $\frac{\partial f_j(B)}{\partial B(i)}$ at $\hat{B}_i$ depends on the whether $v_j(i) = 0$ or not.
If $v_j(i) = 0$, then:
\begin{align*}
\frac{\partial f_j(B)}{\partial B(i)} = \mu^2 \sum_{s=1}^k {d-k-1 \choose s-1}{k \choose k-s} \exp(-2\tau\mu^2s^2/d),
\end{align*}
and if $v_j(i) = \mu^2$ then:
\begin{align*}
\frac{\partial f_j(B)}{\partial B(i)} = \mu^2 \sum_{s=1}^k {d-k \choose s}{k-1 \choose k-s} \exp(-2\tau\mu^2s^2/d).
\end{align*}
Both of these follow from straightforward counting arguments. 
Notice that the value of the subgradient depends only on whether $v_j(i) = 0$ or not, and under the uniform distribution $\pi$, $\EE_{j \sim \pi}v_j(i) = \EE_{j \sim \pi} v_j(i')$.
This implies that the constant vector is in the subgradient of $f(B)$ at $\hat{B}$, so that $\hat{B}$ is the minimizer of $W(\Vcal, \alpha, B)$ subject to $\|B\|_1 \le \tau$. 

We have already done the requisite calculation to bound the minimax risk under sampling.
The calculations above show that if $\mu = \omega(\sqrt{\frac{d}{\tau} \log(k(d-k))})$ then the maximum likelihood estimator, when using the uniform sampling strategy has risk tending to zero.
Conversely if $\mu = o(\sqrt{\frac{d}{\tau} \log(k(d-k))})$ then the minimax risk, for \emph{any} allocation strategy tends to one. 

\textbf{Calculation for Biclusters:}
Due to symmetry, all hypotheses achieve the maximum and therefore, we can directly calculate $W(\Vcal, \alpha)$.
We use the notation $C_n^i$ to denote the binomial coefficient ${n \choose i}$.
\begin{align*}
W(\Vcal, \alpha) &= \sum_{s_r=1}^k \sum_{s_c=1}^k C_{k}^{s_r} C_{k}^{s_c} C_{d-k}^{s_r} C_{d-k}^{s_c} \exp\left(-\frac{2\mu^2}{\alpha}(s_r(k-s_c) + s_c(k-s_r) + s_rs_c)\right)\\
& + \sum_{s_r=1}^k C_k^{s_r} C_{d-k}^{s_r} \exp\left(-\frac{2\mu^2}{\alpha}(s_rk)\right) + \sum_{s_c=1}^k C_k^{s_c} C_{d-k}^{s_c} \exp\left(-\frac{2\mu^2}{\alpha}(s_ck)\right)\\
\end{align*}
This last two term comes from the case where $s_c = 0$ or $s_r = 0$, which is all of the hypotheses that share the same columns but disagree on the rows (or share the same rows but disagree on the columns). 
Using binomial approximations, the first term can be upper bounded by:
\begin{align*}
& \le \sum_{s_r=1}^k\sum_{s_c=1}^k \exp\left(s_r\log\frac{k(d-k)e^2}{s_r^2} + s_c\log\frac{k(d-k)e^2}{s_c^2} - \frac{2\mu^2}{\alpha}(s_r(k - s_c/2) + s_c(k-s_r/2))\right)\\
& \le \sum_{s_r=1}^k\exp\left(s_r\left(\log\frac{k(d-k)e^2}{s_r^2} - \frac{k\mu^2}{\alpha}\right)\right)\sum_{s_c=1}^k\exp\left(s_c\left(\log\frac{k(d-k)e^2}{s_c^2} - \frac{k\mu^2}{\alpha}\right)\right).
\end{align*}
The two terms here are identical, so we will just bound the first one:
\begin{align*}
&\sum_{s_r=1}^k\exp\left(s_r\left(\log\frac{k(d-k)e^2}{s_r^2} - \frac{k\mu^2}{\alpha}\right)\right)\\
& \le \sum_{s_r=1}^k \exp\left(s_r \left(\log(k(d-k)e^2) - k\mu^2/\alpha\right)\right)\\
& \le k \exp\left(\log(k(d-k)e^2) - k\mu^2/\alpha\right) \qquad \textrm{ if } \mu^2 \ge \frac{\alpha}{k}\log(k(d-k)e^2)
\end{align*}
Applying this inequality to both terms gives a bound on $W(\Vcal, \alpha)$. 
This bound is smaller than $\delta$ as long as $\mu \ge \sqrt{\frac{c}{k\alpha}\log(k(d-k)e/\delta)}$ for some universal constant $c$. 
Again this subsumes the condition required for the inequality to hold. 

The other two terms are essentially the same.
Using binomial approximations, both expressions can be bounded as:
\begin{align*}
& \sum_{s_r=1}^k C_k^{s_r} C_{d-k}^{s_r} \exp\left(-\frac{2\mu^2}{\alpha}(s_rk)\right) = \sum_{s_r=1}^k \exp(s_r \log(e^2k(d-k)/s_r^2) - 2s_rk\mu^2/\alpha)\\
& \le k \exp(\log (k(d-k)e^2) - 2k\mu^2/\alpha) \qquad \textrm{ if } \mu^2 \ge \frac{\alpha}{2k}\log(k(d-k)e^2).
\end{align*}
These bounds lead to the same minimax rate as above.

For the lower bound, we again use binomial approximations. 
\begin{align*}
W(\Vcal,\alpha) &\ge \sum_{s_r=1}^k\sum_{s_c=1}^k \exp\left(s_r\log\frac{k(d-k)}{s_r^2} + s_c\log\frac{k(d-k)}{s_c^2} - \frac{2\mu^2}{\alpha}(s_r(k - s_c/2) + s_c(k-s_r/2))\right)\\
& \ge \sum_{s_r=1}^k\exp\left(s_r\left(\log\frac{k(d-k)e^2}{s_r^2} - \frac{2k\mu^2}{\alpha}\right)\right)\sum_{s_c=1}^k\exp\left(s_c\left(\log\frac{k(d-k)e^2}{s_c^2} - \frac{2k\mu^2}{\alpha}\right)\right)\\
& \ge \exp(\log(k(d-k) - 2\mu^2k/\alpha)^2
\end{align*}
This lower bound goes to infinity if $\mu = o(\sqrt{\frac{1}{k}\log(k(d-k))})$ lower bounds the minimax rate.

To certify that the uniform allocation strategy minimizes $W(\Vcal, \alpha, B)$, we apply Proposition~\ref{prop:optimal_sampling}. 
Fix $\tau$ and let $\hat{B}$ be such that $\hat{B}((a,b)) = \tau/d^2$ for all $(a,b) \in [d] \times [d]$. 
By symmetry, every hypothesis achieves the maximum under this allocation strategy, and we will take $\pi$ to be the uniform distribution over all hypothesis.

For a hypothesis $j$, let $f_j(B)$ denote the term in the SEDF centered around $j$.
For a hypothesis $j$ based on clusters $S_l, S_r$ and a coordinate $(a,b)$, the subgradient $\frac{\partial f_j(B)}{\partial B(a,b)}$ at $\hat{B}(a,b)$ depends on whether $a \in S_l$ and $b \in S_r$. 
If $a \notin C_l$ and $b \notin C_r$, then:
{\small 
\begin{align*}
\left.\frac{\partial f_j(B)}{\partial B(a,b)}\right|_{B = \hat{B}} = \frac{-\mu^2}{\alpha} \sum_{s_r=1}^k\sum_{s_c=1}^k C_{d-k-1}^{s_r-1} C_{k}^{s_r} C_{d-k-1}^{s_c-1} C_{k}^{s_c} \exp(\frac{-2\tau\mu^2}{\alpha d^2}\left( s_r(k-s_c/2) + s_c(k-s_r/2)\right)).
\end{align*}
}
This follows by direct calculation. 
Similar calculations yield the other cases:
{\small 
\begin{align*}
\left.\frac{\partial f_j(B)}{\partial B(a,b)}\right|_{B = \hat{B}} &= \frac{-\mu^2}{\alpha} \sum_{s_r=0}^{k-1}\sum_{s_c=1}^k C_{d-k}^{s_r} C_{k-1}^{s_r} C_{d-k-1}^{s_c-1} C_{k}^{s_c} \exp(\frac{-2\tau\mu^2}{\alpha d^2}\left( s_r(k-s_c) + s_c(k-s_r) + s_rs_c \right)).\\
\left.\frac{\partial f_j(B)}{\partial B(a,b)}\right|_{B = \hat{B}} &= \frac{-\mu^2}{\alpha} \sum_{s_r=1}^{k}\sum_{s_c=0}^{k-1} C_{d-k-1}^{s_r-1} C_{k}^{s_r} C_{d-k}^{s_c} C_{k-1}^{s_c} \exp(\frac{-2\tau\mu^2}{\alpha d^2}\left( s_r(k-s_c) + s_c(k-s_r) + s_rs_c \right)).\\
\left.\frac{\partial f_j(B)}{\partial B(a,b)}\right|_{B = \hat{B}} &= \frac{-\mu^2}{\alpha} \sum_{s_r=0}^{k-1} \sum_{s_c=0}^{k-1} C_{d-k}^{s_r} C_{k-1}^{s_r} C_{d-k}^{s_c} C_{k-1}^{s_c} \exp(\frac{-2\tau\mu^2}{\alpha d^2}\left( s_r(k-s_c) + s_c(k-s_r) + s_rs_c \right)).
\end{align*}
}
These correspond to the cases $a \in S_l, b \notin S_r$, $a \notin S_l, b \in S_r$ and the case where $a \in S_l, b \in S_r$ respectively.
The main point is that the value of the subgradient depends only on presence or absence of the row/column in the cluster, and under the uniform distribution $\pi$, each row/column is equally likely to be in the cluster.
This means that for every coordinate $(a,b)$ taking the expected subgradient with respect to the uniform distribution over hypotheses yields the same expression.
So the constant vector is in the subgradient of $f(B)$ at $\hat{B}$, so that $\hat{B}$ is the minimizer of $W(\Vcal, \alpha, B)$ subject to $\|B\|_1 \le \tau$. 

We have already done the requisite calculation to bound the minimax risk under sampling.
The calculations above show that if $\mu = \omega(\sqrt{\frac{d^2}{k \tau} \log(k(d-k))})$ then the maximum likelihood estimator, when using the uniform sampling strategy has risk tending to zero.
Conversely if $\mu = o(\sqrt{\frac{d^2}{k \tau} \log(k(d-k))})$ then the minimax risk, for \emph{any} allocation strategy tends to one. 

The biclusters family is clearly unitarily invariant with respect to the set of orthonormal matrices that permute the rows and columns independently. 
The family is easiest to describe as acting on the matrices $\mathbf{1}_{S_l} \mathbf{1}_{S_r}^T$.
Let $P_l, P_r$ be any two $d \times d$ permutation matrices.
Then the matrix $P_l \mathbf{1}_{S_l} (\mathbf{1}_{S_r}P_r)^T$ is clearly another hypothesis, and as we vary $P_l$ and $P_r$ we generate all of the hypothesis.
Note that these permutations are unitary operators on the matrix space $\RR^{d \times d}$, which allows us to apply Theorem~\ref{thm:gaussian_rotation}.

For the analysis of the interactive algorithm, let us first bound the probability that the algorithm makes a mistake on any single coordinate. 
Consider sampling a coordinate $x$ with mean $\mu$ and noise variance $1/b$.
A Gaussian tail bound reveals that:
\begin{align*}
\PP[|x-\mu| \ge \epsilon] \le 2\exp(-2b\epsilon^2).
\end{align*}
We will sample no more than $d^2$ coordinates and we will sample each coordinate with the same amount of energy $b$.
So by the union bound, the probability that we make a single mistake in classifying a coordinate that we query is bounded by $\delta/2$ as long as:
\begin{align*}
\mu \ge 2\epsilon = \sqrt{\frac{2}{b}\log(4d^2/\delta)}.
\end{align*}

We now need to bound $b$, which depends on the total number of coordinates queried by the algorithm.
In the first phase of the algorithm, we sample coordinates uniformly at random until we hit one that is active.
Since each sample hits an active coordinate with probability $k^2/d^2$:
\begin{align*}
\PP[\textrm{hit active coordinate in $T$ samples}] = 1 - (1-k^2/d^2)^T \ge 1 - \frac{1}{e^{Tk^2/d^2}},
\end{align*}
or if $T = \frac{d^2}{k^2}\log(2/\delta)$, the probability that we hit an active coordinate in $T$ samples will be at least $1-\delta/2$. 
The total number of samples we use then can be upper bounded by $2d+\frac{d^2}{k^2}\log(2/\delta)$, which means that we can allocate our budget $\tau$ evenly over these coordinates.
Therefore we can set $b = \tau(2d+\frac{d^2}{k^2}\log(2/\delta))^{-1}$, and plugging into the condition on $\mu$ above proves the result.

\textbf{Calculations for Hierarchical Clustering:}
We first describe the heirarchical clustering model.

We study the special case of balanced binary hierarchical clustering on $n$ objects which, without loss of generality we call $[n] = \{1, \ldots, n\}$.
A binary hierarchical clustering is a collection $\Ccal$ of subset of $[n]$, such $[n] \in \Ccal$, and each $C_{\xi} \in \Ccal$, if $|C_\xi| > m$, then there exists two sets $C_{\xi \circ L}, C_{\xi \circ R} \in \Ccal$ both of size $C_\xi/2$ that partition $C_\xi$. 
As a naming convention, we identify a cluster by a string $\xi$ of $L$ and $R$ symbols.
The two sub-clusters of a non-terminal cluster $C_\xi$ are $C_{\xi \circ L}$ and $C_{\xi \circ R}$.
The noisy Constant Block Model is defined using this terminology as follows.
\begin{definition}{\cite{balakrishnan2011noise}}
A similarity matrix $W$ is a {\bf noisy constant block matrix} (noisy CBM) if $W \triangleq A +R$ where $A$ is ideal and $R$ is a perturbation matrix:
\begin{itemize}
\item{} An {\bf ideal similarity matrix} is characterized by off-block diagonal similarity values $\beta_\xi \in [0,1]$ for each cluster $C_\xi$  such that if $x \in C_{\xi \circ L}$ and $y \in C_{\xi \circ R}$, where $C_{\xi\circ L}$ and $C_{\xi \circ R}$ are two sub-clusters of $C_\xi$ at the next level in a binary hierarchy, then $A_{x,y} = \beta_{\xi}$.
Additionally, $\min\{\beta_{\xi \circ R}, \beta_{\xi \circ L}\} \ge \beta_{\xi}$.
Define $\mu = \min\{\min_{\xi}\{\min\{\beta_{\xi \circ R}, \beta_{\xi \circ L}\} - \beta_{\xi}\}, \beta_0\}$, where $\beta_0$ is the minimum overall similarity.
\item{} A symmetric $(n\times n)$ matrix $R$ is a {\bf perturbation matrix}  with parameter $\sigma$ if (a) $\mathbb{E}(R_{ij}) = 0$, (b) the entries of  $R$ are subgaussian, that is $\mathbb{E}(\exp(tR_{ij})) \le \exp \left(\frac{\sigma^2 t^2}{2}\right)$ and (c) for each row $i$, $R_{i1}, \ldots  R_{in}$ are independent.
\end{itemize}
\end{definition}
We focus on a subfamily of this model, parameterized by $n,m,\mu$ where both $n$ and $m$ are powers of two. 
Our subfamily $\Hcal$ consists of all perfectly balanced hierarchical clusterings on $n$ objects with minimum cluster size $m$ and where all similarities are an integral multiple of $\mu$. 
This is the simple hierarchical clustering model specified in Section~\ref{sec:gaussian_examples}.
Note that this set of model is a subset of the noisy CBM, so a lower bound for this family applies to the noisy CBM.
Let $\Vcal$ denote the class of all such matrices, parameterized by number objects $n$, minimum cluster size $m$, and signal strength $\mu$. 
We interpret $\Vcal$ as a collection of vectors defined by $v_{\Ccal} = \textrm{vec}(M[\Ccal])$ for each perfectly balanced hierarchical clustering $\Ccal$.

We now prove Corollary~\ref{cor:hcluster}.
For the first claim, by Theorem~\ref{thm:gaussian_minimax_bd}, we must lower bound the quantity $W(\Vcal, \alpha)$,
\begin{align*}
W(\Vcal, \alpha) = \max_{\Ccal \in \Hcal} \sum_{\Ccal' \ne \Ccal} \exp\left( \|v_{\Ccal} - v_{\Ccal'}\|_2^2/\alpha\right)
\end{align*}

Let $\Ccal_0$ be one of these models.
Consider perturbing $\Ccal_0$ by taking an object and swapping that object with another one in the adjacent cluster at the deepest level of the hierarchy.
There are $nm/2$ such perturbations and any perturbation $\Ccal$ has $\|v_{\Ccal_0} - v_\Ccal\|_2^2 = \mu^2 (8m - 4)$.
This gives the lower bound of:
\begin{align*}
W(\Vcal, \alpha) \ge \frac{nm}{2}\exp\left(\frac{\mu^2}{\alpha} (8m-4)\right)
\end{align*}
By Theorem~\ref{thm:gaussian_minimax_bd}, if $W(\Vcal,1) \ge 3$, then the minimax risk is bounded from above by $1/2$.
Applying our lower bound and solving for $\mu$ proves the first part of the result.

For the second claim, if we certify that the uniform sampling strategy minimizes $W(\Hcal, \alpha, B)$ under the budget constraint, then we can immediately apply Theorem~\ref{thm:gaussian_noninteractive_bd}.
We will use Proposition~\ref{prop:optimal_sampling} to achieve this.

It is easy to see that for the class $\Vcal$, when $\hat{B}$ is uniform, every one of these hypotheses achieves the maximum in the definition of $W(\Vcal, \alpha, B)$.
Moreover, notice that for every pair of pairs objects $\{a,b\}, \{a',b'\}$, there is a bijection $p$ over $\Vcal$ based on swapping $a$ with $a'$ and $b$ with $b'$ in the hierarchy such that for any hypothesis $v_\Ccal$, we have $v_{\Ccal}(a,b) = v_{p(\Ccal)}(a',b')$.
If in $\Ccal$, $a$ and $b$ are clustered at some level $l$, then by swapping $a$ with $a'$ and $b$ with $b'$ to form $p(\Ccal)$, $a'$ and $b'$ are clustered at level $l$ in $p(\Ccal)$ so both terms will be identical because we are in a constant block model.

Since $p$ is a bijection, when we take $\pi$ to be uniform over the hypotheses, we have:
\begin{align*}
& \EE_{\Ccal \sim \pi} \sum_{\Ccal' \ne \Ccal} (v_{\Ccal}(a,b) - v_{\Ccal'}(a,b))^2 \exp(-\|v_{\Ccal'} - v_{\Ccal}\|_2^2)\\
& = \EE_{\Ccal \sim \pi} \sum_{\Ccal' \ne \Ccal} (v_{p(\Ccal)}(a', b') - v_{p(\Ccal')}(a', b')^2 \exp(-\|v_{p(\Ccal')} - v_{p(\Ccal)}\|_2^2)\\
& = \EE_{\Ccal \sim \pi} \sum_{\Ccal' \ne \Ccal} (v_{\Ccal}(a',b') - v_{\Ccal'}(a',b'))^2 \exp(-\|v_{\Ccal'} - v_{\Ccal}\|_2^2).
\end{align*}
This means we may apply Proposition~\ref{prop:optimal_sampling}, which certifies that the uniform sampling minimizes the function $W(\Hcal',\alpha, B)$ under budget constraint.

Equipped with this fact, we can reproduce the calculation above but with $B_i = \tau/{n \choose 2}$, giving:
\begin{align*}
W(\Vcal, \alpha, \tau) \ge \frac{nm}{2}\exp\left(\frac{\mu^2}{\alpha}\frac{\tau}{{n \choose 2}} (8m-4)\right)
\end{align*}

This class is also unitarily invariant using the same tensorized permutation family from the biclustering example. 
Therefore the MLE is optimal.\qed

\textbf{Calculation for Stars:}
For the stars problem, define $\textrm{Nb}(j) \subset V$ to be the neighbors of the vertex $j$ in the graph.
For a fixed hypothesis $j$, we have 
\begin{align*}
W_j(\Vcal, \alpha) &= \sum_{k \ne j} \exp\left( - \|v_k-v_k\|_2^2/\alpha\right)\\
& = \sum_{k \in \textrm{Nb}(j)} \exp(-\mu^2(\textrm{deg}(k) + \textrm{deg}(j) - 2)/\alpha) + \sum_{k \notin \textrm{Nb}(j)} \exp(-\mu^2(\textrm{deg}(k)+\textrm{deg}(j))/\alpha)\\
& \le \exp\left(-\mu^2\textrm{deg}_{\min}/\alpha - \mu^2\textrm{deg}(j)/\alpha\right)\left( \textrm{deg}(j)\exp(2\mu^2/\alpha) + |V| - \textrm{deg}(j)\right)
\end{align*}
This last inequality follows by replacing every $\textrm{deg}(k)$ with $\textrm{deg}_{\min}$, the lower bound on the degrees.
This last expression is maximized with $\textrm{deg}(j) = \textrm{deg}_{\min}$, which can be observed by noticing that the derivative with respect to $\textrm{deg}(j)$ is negative. 
This gives the bound:
\begin{align*}
W(\Vcal, \alpha) \le \exp\left(-2\mu^2\textrm{deg}_{\min}/\alpha\right)\left(\textrm{deg}_{\min}\exp(2\mu^2/\alpha) + |V| - \textrm{deg}_{\min}\right)
\end{align*}

One can lower bound $W(\Vcal,\alpha)$ by choosing the hypothesis $j$ with $\textrm{deg}(j) = \textrm{deg}_{\min}$ and then replacing all other degree terms with $\textrm{deg}_{\max}$ in the above calculations.
This gives:
\begin{align*}
W(\Vcal, \alpha) \ge \exp\left(-\frac{\mu^2}{\alpha}(\textrm{deg}_{\min} + \textrm{deg}_{\max})\right)\left(\textrm{deg}_{\min}e^{2\mu^2/\alpha} + |V| - \textrm{deg}_{\min}\right)
\end{align*}

\end{document}